\documentclass{article}

\usepackage[utf8]{inputenc} 
\usepackage[T1]{fontenc}    
\usepackage{hyperref}       
\usepackage{url}            
\usepackage{booktabs}       
\usepackage{amsfonts}       
\usepackage{nicefrac}       
\usepackage{microtype}      
\usepackage{authblk}
\usepackage{fullpage}
\usepackage{parskip}

\def\arxiv{1}

\usepackage{amssymb}
\usepackage{amsmath,amsthm}
\usepackage{mathtools}
\usepackage{stmaryrd}
\usepackage{subfigure}
\usepackage{algorithm}
\usepackage{algorithmic}
\usepackage{color}
\usepackage{hyperref}
\usepackage{natbib}

\usepackage{bbm}

\usepackage{enumitem}

\newcommand\myexp{\mathop{{}\mathbb{E}}}

\DeclareMathOperator{\range}{range}
\DeclareMathOperator*{\argmin}{arg\,min}

\newcommand{\costmp}{\mathcal Q_{\mathit{MP}}} 
\newcommand{\costuni}{\mathcal Q_{\Loss}} 
\newcommand{\R}{\mathbb{R}}             
\newcommand{\loc}{\mathit{loc}}         
\newcommand{\loss}{\ell}                   
\newcommand{\Loss}{\mathcal{L}}                   
\newcommand{\intset}[1]{\llbracket #1\rrbracket}

\newcommand{\ttheta}{\widetilde{\Theta}}
\newcommand{\Nei}[1]{\mathcal{N}_{#1}}

\newtheorem{theorem}{Theorem} 
\newtheorem{lemma}{Lemma}
\newtheorem{proposition}{Proposition}

\newtheorem{remark}{Remark}

\title{Personalized and Private Peer-to-Peer Machine Learning}

\author[1]{Aurélien Bellet\thanks{\texttt{first.last@inria.fr}}}
\author[2]{Rachid Guerraoui\thanks{\texttt{first.last@epfl.ch}}}
\author[2]{Mahsa Taziki$^\dagger$}
\author[3]{Marc Tommasi$^\star$}
\affil[1]{INRIA}
\affil[2]{EPFL}
\affil[3]{Université de Lille}
\date{}

\begin{document}

\maketitle


\begin{abstract}
The rise of connected personal devices together with privacy concerns call for machine learning algorithms capable of leveraging the data of a large number of agents to learn personalized models under strong privacy requirements. In this paper, we introduce an efficient algorithm to address the above problem in a fully decentralized (peer-to-peer) and asynchronous fashion, with provable convergence rate. We show how to make the algorithm differentially private to protect against the disclosure of information about the personal datasets, and formally analyze the trade-off between utility and privacy. Our experiments show that our approach dramatically outperforms previous work in the non-private case, and that under privacy constraints, we can significantly improve over models learned in isolation.
\end{abstract}


\section{Introduction}

Connected personal devices are now widespread: they can collect and process increasingly large and sensitive user data. As a concrete example, consider the health domain. Smart watches can record cardiac activities, mobile apps encourage users to participate to studies (about depression, concussion, etc.),\footnote{See e.g., \url{https://www.apple.com/researchkit/}} and recent painless sensors can replace a finger prick for blood glucose testing~\citep{electronics6030065}.
Such information can be leveraged through machine learning to provide useful personalized services (e.g., personalized treatments) to the user/patient. A common practice is to centralize data from all devices on an external server for batch processing, sometimes without explicit consent from users and with little oversight. While this data concentration is ideal for the utility of the learning process, it raises serious privacy concerns and opens the door to potential misuse (e.g., exploitation for the purpose of recruitment, insurance pricing or granting loans).
Therefore, in applications where the data is considered too sensitive to be shared (due to legislation or because the user opts out), one has to learn on each device separately without taking advantage of the multiplicity of data sources (e.g., information from similar users). This preserves privacy but leads to poor accuracy, in particular for new or moderately active users who have not collected much data.

Instead of the above two extreme approaches, our goal is to design a solution allowing a large number of users (agents) to collaborate so as to learn more accurate personalized models while ensuring that their data stay on their local device and that the algorithm does not leak sensitive information to others. We consider a \emph{fully decentralized} solution where agents operate asynchronously and communicate over a network in a peer-to-peer fashion, without any central entity to maintain a global state of the system or even to coordinate the protocol. The network  acts as a communication network but also models similarities between users.
While a decentralized architecture may be the only available option in some applications (e.g., IoT), it also provides interesting benefits when a more traditional distributed (master/slave) architecture could be used. In particular, peer-to-peer algorithms provide scalability-by-design to large sets of devices thanks to the locality of their updates \citep{Kermarrec2015a}. For instance, it was recently shown that fully decentralized learning algorithms can perform better than their distributed counterparts because they avoid a communication bottleneck at the master node \citep{Lian2017b}. Finally, a decentralized architecture intrinsically provides some security guarantees as it becomes much more difficult for any party (or any external adversary) to observe the full state of the system.

The problem of decentralized collaborative learning of personal models has been recently considered by \citet{Vanhaesebrouck2017a}, but they did not consider any privacy constraints.
In fact, while there has been a large body of work on privacy-preserving machine learning from centralized data, notably based on differential privacy \citep[see][and references therein]{Dwork2014a,regularized,Bassily2014a}, the case where sensitive datasets are distributed across multiple data owners has been much less studied, let alone the fully decentralized setting. Existing approaches for privacy-preserving distributed learning \citep[see e.g.,][]{Pathak2010a,Rajkumar2012a,Shokri2015a,Huang2015a} rely on a central (sometimes trusted) server, assume the local data distribution is the same for all users and/or are designed to learn a single global model rather than a personal model for each user.

In this paper, we ask a challenging question: given the above decentralization and privacy constraints, can agents improve upon their purely local models through collaboration?
Our contributions towards a positive answer to this question are three-fold.
First, we propose a decentralized and asynchronous block coordinate descent algorithm to address the problem in the non-private setting. Taking advantage of the structure of the problem, this algorithm has simple updates and provable convergence rates, improving upon the previous work of \citet{Vanhaesebrouck2017a}.
Second, we design a differentially-private scheme based on randomly perturbing each update of our algorithm. This scheme guarantees that the messages sent by the users over the network during the execution of the algorithm do not reveal significant information about any data point of any local dataset. We provide a formal analysis of the utility loss due to privacy.
Third, we conduct experiments on synthetic and real-world data to validate our approach. The empirical results show that the trade-off between utility and privacy is in line with our theoretical findings, and that under strong privacy constraints we can still outperform the purely local models in terms of accuracy.

The rest of the paper is organized as follows. 
In Section~\ref{sec:cd_algo}, we describe the problem setting and presents our decentralized algorithm for the non-private case. Section~\ref{sec:privacy} introduces a differentially private version and analyzes the trade-off between utility and privacy. In
Section~\ref{sec:related}, we discuss some related work on decentralized and private learning.
Finally, Section~\ref{sec:exp} is dedicated to numerical experiments. Detailed proofs can be found in the supplementary material.

\section{Peer-to-Peer Personalized  Learning with Coordinate Descent}
\label{sec:cd_algo}

We start by formally describing the learning problem that we address in this paper.

\subsection{Problem Setting}

We consider a set of $n$ agents. Each agent $i$ has a local data distribution $\mu_i$ over the space $\mathcal{X}\times\mathcal{Y}$ and has access to a set $\mathcal{S}_i=\{(x_i^j,y_i^j)\}_{j=1}^{m_i}$ of $m_i\geq 0$ training examples drawn i.i.d. from $\mu_i$. The goal of agent $i$ is to learn a model $\theta\in\R^p$ with small expected loss $\myexp_{(x_i,y_i)\sim\mu_i}[\loss(\theta;x_i,y_i)]$, where the loss function $\loss(\theta;x_i,y_i)$ is convex in $\theta$ and measures the performance of $\theta$ on data point $(x_i,y_i)$. In the setting where agent $i$ must learn on its own, a standard approach is to select the model minimizing the local (potentially regularized) empirical loss:
\begin{equation}
\label{eq:solitary}
\theta^{\loc}_i \in \argmin_{\theta\in\R^p} \Big[\underbrace{\frac{1}{m_i}\sum_{j=1}^{m_i} \loss(\theta; x_i^j,y_i^j) + \lambda_i\|\theta\|^2}_{\Loss_i(\theta; \mathcal{S}_i)}\Big],
\end{equation}
with $\lambda_i\geq 0$.
In this paper, agents do not learn in isolation but rather participate in a decentralized peer-to-peer network over which they can exchange information. Such collaboration gives them the opportunity to learn a better model than \eqref{eq:solitary}, for instance by allowing some agents to compensate for their lack of data.
Formally, let $\intset{n}=\{1,\dots,n\}$ and $G=(\intset{n},E,W)$ be a weighted connected graph over the set of agents where $E\in\intset{n}\times\intset{n}$ is the set of edges and $W\in\R^{n\times n}$ is a nonnegative weight matrix. $W_{ij}$ gives the weight of edge $(i,j)\in E$ with the convention that  $W_{ij}=0$ if $(i,j)\notin E$ or $i=j$. Following previous work \citep[see e.g.,][]{Evgeniou2004a,Vanhaesebrouck2017a}, we assume that the edge weights reflect a notion of ``task relatedness'': the weight $W_{ij}$ between agents $i$ and $j$ tends to be large if the models minimizing their respective expected loss are similar. These pairwise similarity weights may be derived from user profiles (e.g., in the health domain: weight, size, diabetes type, etc.) or directly from the local datasets, and can be computed in a private way \citep[see e.g.,][]{Goethals2004a,alaggan:hal-00646831}.

In order to scale to large networks, our goal is to design \emph{fully decentralized algorithms}: each agent $i$ only communicates with its neighborhood $\Nei{i}=\{j : W_{ij}>0\}$ without global knowledge of the network, and operates without synchronizing with other agents.
Overall, the problem can thus be seen as a multi-task learning problem over a large number of tasks (agents) with imbalanced training sets, and which must be solved in a fully decentralized way.


\subsection{Objective Function}
\label{sec:obj}

Our goal is to jointly learn the models of the agents by leveraging both their local datasets and the similarity information embedded in the network graph. Following a well-established principle in the multi-task learning literature \citep{Evgeniou2004a,Maurer2006a,Dhillon2011b}, we use graph regularization to favor models that vary smoothly on the graph. Specifically, representing the set of all models $\Theta_i\in\mathbb{R}^p$  as a stacked vector $\Theta = [\Theta_1;\dots;\Theta_n]\in\mathbb{R}^{np}$, the objective function we wish to minimize is given by
\begin{equation}
    \costuni(\Theta) = \frac{1}{2} \sum_{i<j}^n W_{ij} {\lVert \Theta_i - \Theta_j \rVert}^2 + \mu \sum_{i=1}^n D_{ii}c_i \Loss_i(\Theta_i; \mathcal{S}_i),
    \label{eq:Qunified}
\end{equation}
where $\mu>0$ is a trade-off parameter, $D_{ii}=\sum_{j=1}^nW_{ij}$ is a normalization factor and $c_i\in(0,1]\propto m_i$ is the ``confidence'' of agent $i$.\footnote{In practice we will set $c_i=m_i/\max_j m_j$ (plus some small constant when $m_i=0$).}
Minimizing~\eqref{eq:Qunified} implements a trade-off between having similar models for strongly connected agents and models that are accurate on their respective local datasets (the higher the confidence of an agent, the more importance given to the latter part). This allows agents to leverage relevant information from their neighbors --- it is particularly salient for agents with less data which can gain useful knowledge from better-endowed neighbors without ``polluting'' others with their own inaccurate model.
Note that the objective \eqref{eq:Qunified} includes the two extreme cases of learning purely local models as in \eqref{eq:solitary} (when $\mu\rightarrow\infty$) and learning a single global model (for $\mu\rightarrow 0$).

We now discuss a few assumptions and properties of $\costuni$. We assume that for any $i\in\intset{n}$, the local loss function $\Loss_i$ of agent $i$ is convex in its first argument with $L_i^{loc}$-Lipschitz continuous gradient. This implies that $\costuni$ is convex in $\Theta$.\footnote{This follows from the fact that the first term in \eqref{eq:Qunified} is a Laplacian quadratic form, hence convex in $\Theta$.} If we further assume that each $\Loss_i$ is $\sigma_i^{loc}$-strongly convex with $\sigma_i^{loc}>0$ (this is the case for instance when the local loss is L2-regularized), then $\costuni$ is $\sigma$-strongly convex with $\sigma \geq \mu\min_{1\leq i\leq n}[D_{ii}c_i\sigma_i^{loc}] > 0$. In other words, for any $\Theta,\Theta'\in\mathbb{R}^{np}$ we have $\costuni(\Theta') \geq \costuni(\Theta) + \nabla\costuni(\Theta)^T(\Theta'-\Theta) + \frac{\sigma}{2}\|\Theta'-\Theta\|_2^2.$
The partial derivative of $\costuni(\Theta)$ w.r.t. the variables in $\Theta_i$ is given by
\begin{equation}
\label{eq:partialgrad}
\textstyle[\nabla\costuni(\Theta)]_i = D_{ii}(\Theta_i + \mu c_i \nabla\Loss_i(\Theta_i; \mathcal{S}_i)) - \sum_{j\in\Nei{i}}W_{ij}\Theta_j.
\end{equation}
 We define the matrices $U_i\in\R^{np\times p}$, $i\in\intset{n}$, such that $(U_1,\dots,U_n) = I_{np}$. 
 For $i\in[n]$, the $i$-th \emph{block Lipschitz constant} $L_i$ of $\nabla\costuni(\Theta)$ satisfies $\|[\nabla\costuni(\Theta+U_id)]_i - [\nabla\costuni(\Theta)]_i\| \leq L_i\|d\|$ for any $\Theta\in\mathbb{R}^{np}$ and $d\in\R^p$. It is easy to see that $L_i = D_{ii}(1+\mu c_i L_i^{loc})$. We denote $L_{min} = \min_i L_i$ and $L_{max} = \max_i L_i$.



\subsection{Non-Private Decentralized Algorithm}
\label{sec:alg}

For ease of presentation, we first present a non-private decentralized algorithm. Note that this is interesting in its own right as the proposed solution improves upon the algorithm previously proposed by \citet{Vanhaesebrouck2017a}, see Section~\ref{sec:related} for a discussion.

\textbf{Time and communication models.}
Our goal is to minimize the objective function \eqref{eq:Qunified} in a fully decentralized manner.
Specifically, we operate in the asynchronous time model \citep{boyd2006randomized}: each agent has a \emph{local} clock ticking at the times of a rate 1 Poisson process, and wakes up when it ticks. This is in contrast to the synchronous model where agents wake up jointly according to a global clock (and thus need to wait for everyone to finish each round). As local clocks are i.i.d., we can equivalently consider a single clock which ticks when one of the local clocks ticks. This provides a more convenient way to state and analyze the algorithms in terms of a global clock counter $t$ (which is unknown to the agents).
For communication, we rely on a broadcast-based model \citep{Aysal2009a,Nedic2011a} where agents send messages to all their neighbors at once (without expecting a reply).
This model is very appealing in wireless distributed systems, as sending a message to all neighbors has the same cost as sending to a single neighbor. 

\textbf{Algorithm.} 
We propose a decentralized coordinate descent algorithm to minimize \eqref{eq:Qunified}. We initialize the algorithm with an arbitrary set of local models $\Theta(0)=[\Theta_i(0);\dots;\Theta_n(0)]$. At time step $t$, an agent $i$ wakes up. Two consecutive actions are then performed by $i$:
\begin{itemize}
\item \emph{Update step}: agent $i$ updates its local model based on the most recent information $\Theta_j(t)$ received from its neighbors $j\in\Nei{i}$:
\if\arxiv1
\begin{align}
\label{eq:cdupdate}
\Theta_i(t+1) &= \Theta_i(t) - (1/L_i)[\nabla\costuni(\Theta(t))]_i\\
&= ( 1-\alpha ) \Theta_i(t) + \alpha \big( \textstyle\sum_{j\in\Nei{i}}\frac{W_{ij}}{D_{ii}}\Theta_j(t) - \mu c_i\nabla\Loss_i(\Theta_i(t); \mathcal{S}_i) \big),\nonumber
\end{align}
\else
\begin{align}
\label{eq:cdupdate}
\Theta_i(t+1) &= \Theta_i(t) - (1/L_i)[\nabla\costuni(\Theta(t))]_i\\
&= ( 1-\alpha ) \Theta_i(t) + \alpha \big( \textstyle\sum_{j\in\Nei{i}}\frac{W_{ij}}{D_{ii}}\Theta_j(t)\nonumber\\
&- \mu c_i\nabla\Loss_i(\Theta_i(t); \mathcal{S}_i) \big),\nonumber
\end{align}
\fi
where $\alpha = 1/(1+\mu c_iL_i^{loc})\in(0,1]$.
\item \emph{Broadcast step}: agent $i$ sends its updated model $\Theta_i(t+1)$ to its neighborhood $\Nei{i}$.
\end{itemize}

The update step \eqref{eq:cdupdate} consists in a block coordinate descent update with respect to $\Theta_i$ and only requires agent $i$ to know the models $\Theta_j(t)$ previously broadcast by its neighbors $j\in\Nei{i}$. Note that the agent does not need to know the global iteration counter $t$, hence no global clock is needed. The algorithm is thus fully decentralized and asynchronous.
Interestingly, notice that this block coordinate descent update is adaptive to the confidence level of each agent in two respects: (i) globally, the more confidence, the more importance given to the gradient of the local loss compared to the neighbors' models, and (ii) locally, when $\Theta_i(t)$ is close to a minimizer of the local loss $\mathcal{L}_i$ (which is the case for instance if we initialize $\Theta_i(0)$ to such a minimizer), agents with low confidence will trust their neighbors' models more aggressively than agents with high confidence (which will make more conservative updates).\footnote{This second property is in contrast to a (centralized) gradient descent approach which would use the same constant, more conservative step size (equal to the standard Lipschitz constant of $\costuni$) for all agents.}
This is in line with the intuition that agents with low confidence should diverge more quickly from their local minimizer than those with high confidence.

\textbf{Convergence analysis.}
Under our assumption that the local clocks of the agents are i.i.d., the above algorithm can be seen as a randomized block coordinate descent algorithm \citep{Wright2015a}. It enjoys a fast linear convergence rate when $\costuni$ is strongly convex, as shown in the following result.

\begin{proposition}[Convergence rate]
\label{prop:conv}
For $T>0$, let $(\Theta(t))_{t=1}^T$ be the sequence of iterates generated by the proposed algorithm running for $T$ iterations from an initial point $\Theta(0)\in\R^{np}$. Let $\costuni^\star \in \min_{\Theta \in\mathbb{R}^{np}}\costuni(\Theta)$. When $\costuni$ is $\sigma$-strongly convex, we have:
$$\mathbb{E}\left[\costuni(\Theta(T)) - \costuni^\star\right] \leq \Big( 1 - \frac{\sigma}{nL_{max}} \Big)^T \left( \costuni(\Theta(0)) - \costuni^* \right).$$
\end{proposition}
\begin{proof}
This follows from a slight adaptation of the proof of \citet{Wright2015a} (Theorem~1 therein) to the block coordinate descent case. Note that the result can also be obtained as a special case of our Theorem~\ref{thm:utility} (later introduced in Section~\ref{sec:dpscheme}) by setting the noise scale $s_i(t)=0$ for all $t,i$.
\end{proof}
\begin{remark}
For general convex $\costuni$, an $O(1/t)$ rate can be obtained, see \citet{Wright2015a} for details.
\end{remark}

Proposition~\ref{prop:conv}
shows that each iteration shrinks the suboptimality gap by a constant factor. While this factor degrades linearly with the number of agents $n$, this is compensated by the fact that the number of iterations done in parallel also scales roughly linearly with $n$ (because agents operate asynchronously and in parallel). We thus expect the algorithm to scale gracefully with the size of the network if the number of updates per agent remains constant.
The value $\frac{\sigma}{L_{max}} \geq \frac{\mu\min_{1\leq i\leq n}[D_{ii}c_i\sigma_i^{loc}]}{\max_{1\leq i\leq n}[D_{ii}(1+\mu c_iL_i^{loc})]} > 0$ is the ratio between the lower and upper bound on the curvature of $\costuni$. Focusing on the relative differences between agents and assuming constant $\sigma_i^{loc}$'s and $L_i^{loc}$'s, it indicates that the algorithm converges faster when the degree-weighted confidence of agents is approximately the same. On the other hand, two types of agents can represent a bottleneck for the convergence rate: (i) a high-confidence and high-degree agent (the overall progress is then very dependent on the updates of that particular agent), and (ii) a low-confidence, poorly connected agent (hence converging slowly).

\section{Differentially Private Algorithm}
\label{sec:privacy}

As described above, the algorithm introduced in the previous section has many interesting properties. 
However, while there is no direct exchange of data between agents, the sequence of iterates broadcast by an agent may reveal information about its private dataset through the gradient of the local loss.
In this section, we define our privacy model and introduce an appropriate scheme to make our algorithm differentially private. We study its utility loss and the trade-off between utility and privacy. 

\subsection{Privacy Model}
At a high level, our goal is to prevent eavesdropping attacks. We assume the existence of an adversary who observes all the information sent over the network during the execution of the algorithm, but cannot access the agents' internal memory. We want to ensure that such an adversary cannot learn much information about any individual data point of any agent's dataset. This is a very strong notion of privacy: each agent does not trust any other agent or any third-party to process its data, hence the privacy-preserving mechanism must be implemented at the agent level. Furthermore, note that our privacy model protects any agent against all other agents even if they collude (i.e., share the information they receive).\footnote{We assume a honest-but-curious model for the agents: they want to learn as much as possible from the information that they receive but they truthfully follow the protocol.}

To formally define this privacy model, we rely on the notion of Differential Privacy (DP) \citep{Dwork2006a}, which has emerged as a powerful measure of how much information about any individual entry of a dataset is contained in the output of an algorithm.
Formally, let $\mathcal{M}$ be a randomized mechanism taking a dataset as input, and let $\epsilon>0,\delta\geq 0$. We say that $\mathcal{M}$ is $(\epsilon,\delta)$-differentially private if for all datasets $\mathcal{S}=\{z_1,\dots,z_i,\dots,z_m\},\mathcal{S}'=\{z_1,\dots,z'_i,\dots,z_m\}$ differing in a single data point and for all sets of possible outputs $\mathcal{O}\subseteq\range(\mathcal{M})$, we have:
\begin{equation}
\label{eq:dp}
Pr(\mathcal{M} ( \mathcal{S}) \in \mathcal{O}) \leq e^{\epsilon}Pr(\mathcal{M} ( \mathcal{S}') \in \mathcal{O}) + \delta,
\end{equation}
where the probability is over the randomness of the mechanism. At a high level, one can see \eqref{eq:dp} as ensuring that $\mathcal{M} ( \mathcal{S})$ does not leak much information about any individual data point contained in $\mathcal{S}$.
DP has many attractive properties: in particular it provides strong robustness against background knowledge attacks and does not rely on computational assumptions. The composition of several DP mechanisms remains DP, albeit a graceful degradation in the parameters  \citep[see][for strong composition results]{boosting,composition}. We refer to \citet{Dwork2014a} for more details on DP.

In our setting, following the notations of \eqref{eq:dp}, each agent $i$ runs a mechanism $\mathcal{M}_{i}(\mathcal{S}_i)$ which takes its local dataset $\mathcal{S}_i$ and outputs all the information sent by $i$ over the network during the execution of the algorithm (i.e., the sequence of iterates broadcast by the agent). Our goal is to make $\mathcal{M}_{i}(\mathcal{S}_i)$ $(\epsilon,\delta)$-DP for all agents $i$ simultaneously. Note that learning purely local models \eqref{eq:solitary} is a perfectly private baseline according to the above definition as agents do not exchange any information. Below, we present a way to collaboratively learn better models while preserving privacy.

\subsection{Privacy-Preserving Scheme}
\label{sec:dpscheme}

The privacy-preserving version of our algorithm consists in replacing the update step in \eqref{eq:cdupdate} by the following one (assuming that at time $t$ agent $i$ wakes up):
\if\arxiv1
\begin{equation}
\label{eq:cdupdatepriv}
\widetilde{\Theta}_i(t+1) = ( 1-\alpha ) \widetilde{\Theta}_i(t) + \alpha \big( \textstyle\sum_{j\in\Nei{i}}\frac{W_{ij}}{D_{ii}}\widetilde{\Theta}_j(t) - \mu c_i(\nabla\Loss_i(\widetilde{\Theta}_i(t); \mathcal{S}_i)+ \eta_i(t)) \big),
\end{equation}
\else
\begin{align}
\label{eq:cdupdatepriv}
\begin{aligned}
\widetilde{\Theta}_i(t+1) &=& ( 1-\alpha ) \widetilde{\Theta}_i(t) + \alpha \big( \textstyle\sum_{j\in\Nei{i}}\frac{W_{ij}}{D_{ii}}\widetilde{\Theta}_j(t) \\
&&- \mu c_i(\nabla\Loss_i(\widetilde{\Theta}_i(t); \mathcal{S}_i)+ \eta_i(t)) \big),
\end{aligned}
\end{align}
\fi
where $\eta_i(t)\sim Laplace(0,s_i(t))^p\in\R^p$ is a noise vector drawn from a Laplace distribution with finite scale $s_i(t)\geq0$.\footnote{We use the convention $Laplace(0,0)=0$ w.p. $1$.} The difference with the non-private update is that agent $i$ adds appropriately scaled Laplace noise to the gradient of its local loss $\Loss_i$. It then sends the resulting noisy iterate $\widetilde{\Theta}_i(t+1)$, instead of $\Theta_i(t+1)$, to its neighbors.
Note that for full generality, we allow the noise to potentially depend on the global iteration number $t$, as we will see towards the end of this section that it opens interesting perspectives.


Assume that update \eqref{eq:cdupdatepriv} is run $T_i$ times by agent $i$ within the total $T>0$ iterations across the network. Let $\mathcal{T}_i=\{t_{i}^k\}_{k=1}^{T_i}$ be the set of iterations at which agent $i$ woke up and consider the mechanism $\mathcal{M}_{i}(\mathcal{S}_i) = \{\widetilde{\Theta}_i(t_i+1) : t_i\in\mathcal{T}_i\}$.
The following theorem shows how to scale the noise at each iteration, $s_i(t_i)$, so as to provide the desired overall differential privacy guarantees.
\begin{theorem}[Differential privacy of $\mathcal{M}_{i}$]
\label{thm:privacy}
Let $i\in\intset{n}$ and assume that $\Loss_i(\theta; \mathcal{S}_i) = \frac{1}{m_i}\sum_{k=1}^{m_i} \loss_i(\theta; x_i^k,y_i^k) + \lambda_i\|\theta\|^2$ where $\loss(\cdot; x,y)$ is $L_0$-Lipschitz with respect to the $L_1$-norm for all $(x,y)\in\mathcal{X}\times\mathcal{Y}$. For any $t_i\in\mathcal{T}_i$, let $s_i(t_i)=\frac{2 L_0}{\epsilon_i(t_i) m_i}$ for some $\epsilon_i(t_i)>0$. For any $\bar{\delta}_i\in[0,1]$ and initial point $\widetilde{\Theta}(0)\in\R^{np}$ independent of $\mathcal{S}_i$, the mechanism $\mathcal{M}_{i}(\mathcal{S}_i)$ is $(\bar{\epsilon}_i,\bar{\delta}_i)$-DP with
\if\arxiv1
\begin{multline*}
\textstyle\bar{\epsilon}_i = \min\Big\{ \sum_{t_i = 1}^{T_i} \epsilon_{i}(t_i), \sum_{t_i = 1}^{T_i}\frac{(e^{\epsilon_{i}(t_i)}-1)\epsilon_{i}(t_i)}{e^{\epsilon_{i}(t_i)}+1}\textstyle+\sqrt{\sum_{t_i = 1}^{T_i}2\epsilon_{i}(t_i)^2\log\big(e+\sqrt{\sum_{t_i = 1}^{T_i}\epsilon_{i}(t_i)^{2}}/\bar{\delta_{i}}\big)},\\
\textstyle \sum_{t_i = 1}^{T_i}{\frac{(e^{\epsilon_{i}(t_i)}-1)\epsilon_{i}(t_i)}{e^{\epsilon_{i}(t_i)}+1}+\sqrt{\sum_{t_i = 1}^{T_i}2\epsilon_{i}(t_i)^2\log(1/\bar{\delta_{i}})}}\Big\}.
\end{multline*}
\else
\begin{multline*}
\textstyle\bar{\epsilon}_i = \min\Big\{ \sum_{t_i = 1}^{T_i} \epsilon_{i}(t_i), \sum_{t_i = 1}^{T_i}\frac{(e^{\epsilon_{i}(t_i)}-1)\epsilon_{i}(t_i)}{e^{\epsilon_{i}(t_i)}+1}\\
\textstyle+\sqrt{\sum_{t_i = 1}^{T_i}2\epsilon_{i}(t_i)^2\log\big(e+\sqrt{\sum_{t_i = 1}^{T_i}\epsilon_{i}(t_i)^{2}}/\bar{\delta_{i}}\big)},\\
\textstyle \sum_{t_i = 1}^{T_i}{\frac{(e^{\epsilon_{i}(t_i)}-1)\epsilon_{i}(t_i)}{e^{\epsilon_{i}(t_i)}+1}+\sqrt{\sum_{t_i = 1}^{T_i}2\epsilon_{i}(t_i)^2\log(1/\bar{\delta_{i}})}}\Big\}.
\end{multline*}
\fi
\end{theorem}
\begin{remark}
We can obtain a similar result if we assume $L_0$-Lipschitzness of $\ell$ w.r.t. $L_2$-norm (instead of $L_1$) and use Gaussian noise (instead of Laplace). Details are in the supplementary material.
\end{remark}
Theorem~\ref{thm:privacy} shows that $\mathcal{M}_{i}(\mathcal{S}_i)$ is $(\bar{\epsilon}_i,0)$-DP for $\bar{\epsilon}_i=\sum_{t_i = 1}^{T_i} \epsilon_{i}(t_i)$. One can also achieve a better scaling for $\bar{\epsilon}_i$ at the cost of setting $\bar{\delta}_i>0$ \citep[see][for a discussion of the trade-offs in the composition of DP mechanisms]{composition}.
The noise scale needed to guarantee DP for an agent $i$ is inversely proportional to the size $m_i$ of its local dataset $\mathcal{S}_i$. This is a classic property of DP, but it is especially appealing in our collaborative formulation as the confidence weights $c_i$'s tune down the importance of agents with small datasets (preventing their noisy information to spread) and give more importance to agents with larger datasets (who propagate useful information).
Our next result quantifies how the added noise affects the convergence.
\begin{theorem}[Utility loss]
\label{thm:utility}
For any $T>0$, let $(\widetilde{\Theta}(t))_{t=1}^T$ be the sequence of iterates generated by $T$ iterations of update \eqref{eq:cdupdatepriv} from an initial point $\widetilde{\Theta}(0)\in\R^{np}$.
For $\sigma$-strongly convex $\costuni$, we have:
\if\arxiv1
\begin{equation*}
\mathbb{E}\left[\costuni(\widetilde{\Theta}(T)) - \costuni^\star\right] \leq \Big( 1 - \frac{\sigma}{nL_{max}} \Big)^T \left( \costuni(\widetilde{\Theta}(0)) - \costuni^\star \right) + \frac{1}{nL_{min}}\sum_{t=0}^{T-1}\sum_{i=1}^n\Big( 1 - \frac{\sigma}{nL_{max}} \Big)^{t}\big(\mu D_{ii} c_i s_i(t)\big)^2.
\end{equation*}
\else
\begin{multline*}
\mathbb{E}\left[\costuni(\widetilde{\Theta}(T)) - \costuni^\star\right] \leq \\\Big( 1 - \frac{\sigma}{nL_{max}} \Big)^T \left( \costuni(\widetilde{\Theta}(0)) - \costuni^\star \right)\\
+ \frac{1}{nL_{min}}\sum_{t=0}^{T-1}\sum_{i=1}^n\Big( 1 - \frac{\sigma}{nL_{max}} \Big)^{t}\big(\mu D_{ii} c_i s_i(t)\big)^2.
\end{multline*}
\fi
\end{theorem}
This result shows that the error of the private algorithm after $T$ iterations decomposes into two terms. The first term is the same as in the non-private setting and decreases with $T$. The second term gives an additive error due to the noise, which takes the form of a weighted sum of the variance of the noise added to the iterate at each iteration (note that we indeed recover the non-private convergence rate of Proposition~\ref{prop:conv} when the noise scale is $0$).
When the noise scale used by each agent is constant across iterations, this additive error converges to a finite number as $T\rightarrow\infty$. The number of iterations $T$ rules the trade-off between the two terms. We give more details in the supplementary material and study this numerically in Section~\ref{sec:exp}.

In practical scenarios, each agent $i$ has an overall privacy budget $(\bar{\epsilon}_i,\bar{\delta}_i)$. Assume that the agents agree on a value for $T$ (e.g., using Proposition~\ref{prop:conv} to achieve the desired precision). Each agent $i$ is thus expected to wake up $T_i=T/n$ times, and can use Theorem~\ref{thm:privacy} to appropriately distribute its privacy budget across the $T_i$ iterations and stop after $T_i$ updates. A simple and practical strategy is to distribute the budget equally across the $T_i$ iterations. Yet, Theorem~\ref{thm:utility} suggests that better utility can be achieved if the noise scale increases with time. Assume that agents know in advance the clock schedule for a particular run of the algorithm, i.e. agent $i$ knows the global iterations $\mathcal{T}_i$ at which it will wake up.
The following result gives the noise allocation policy minimizing the utility loss.






\begin{proposition}
\label{prop:allfullprivate}
Let $C=1 - \sigma/nL_{max}$ and for any agent $i\in\intset{n}$ define $\lambda_{\mathcal{T}_i}(i) = \sum_{t \in \mathcal{T}_i} \frac{\sqrt[3]{ C} - 1}{\sqrt[3]{C^T} - 1} \sqrt[3]{C^t}$. Assuming $s_i(t_i)=\frac{2 L_0}{\epsilon_i(t_i) m_i}$ for $t_i\in\mathcal{T}_i$ as in Theorem~\ref{thm:privacy}, the following privacy parameters optimize the utility loss while ensuring the budget $\bar{\epsilon}_i$ is matched exactly:
$$
\epsilon_i(t) = \left\{ \begin{array}{ll}
 & \frac{\sqrt[3]{ C} - 1}{\sqrt[3]{C^T} - 1} \sqrt[3]{C^{t}} \frac{\bar{\epsilon}_i}{\lambda_{\mathcal{T}_i}(i)} \text{ for } t \in \mathcal{T}_i, \\
 & 0 \text{ otherwise.}
\end{array} \right.
$$
\end{proposition}
The above noise allocation policy requires agents to know in advance the schedule and the global iteration counter. This is an unrealistic assumption in the fully decentralized setting where no global clock is available. Still, Proposition~\ref{prop:allfullprivate} may be useful to design heuristic strategies that are practical, for instance, based on using the \emph{expected} global time for the agent to wake up at each of its iterations. We leave this for future work.

\begin{remark}
Theorem~\ref{thm:utility} implies that a good warm start point $\Theta(0)$ is beneficial. However, $\Theta(0)$ must be DP. In the supplementary material, we show how to generate a private warm start based on propagating perturbed versions of purely local models in the network.
\end{remark}

\begin{figure*}[t]
    \centering
    \includegraphics[width=\textwidth]{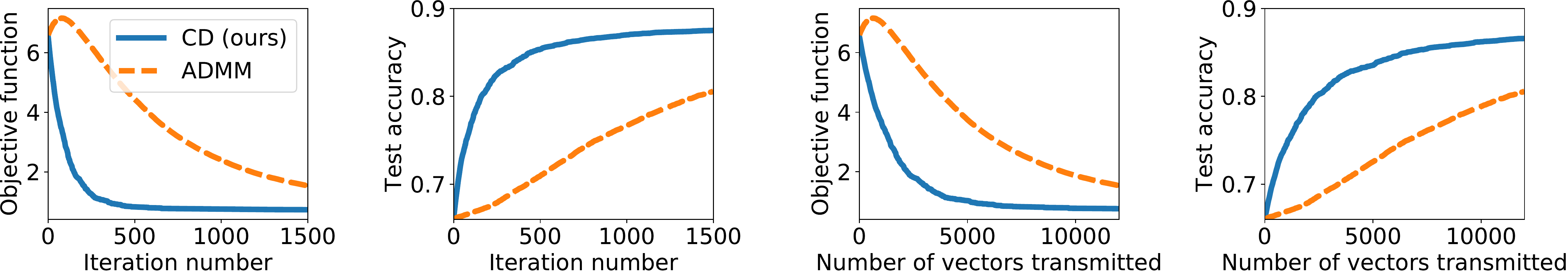}
    \vspace*{-1em}
    \caption{\label{fig:cd_vs_admm}Our Coordinate Descent (CD) algorithm compared to the existing ADMM algorithm.}
\end{figure*}

\begin{figure*}[t]
    \centering
    \subfigure[\label{fig:private:cdcst}Init. with constant vector]{
    \includegraphics[width=0.27\textwidth]{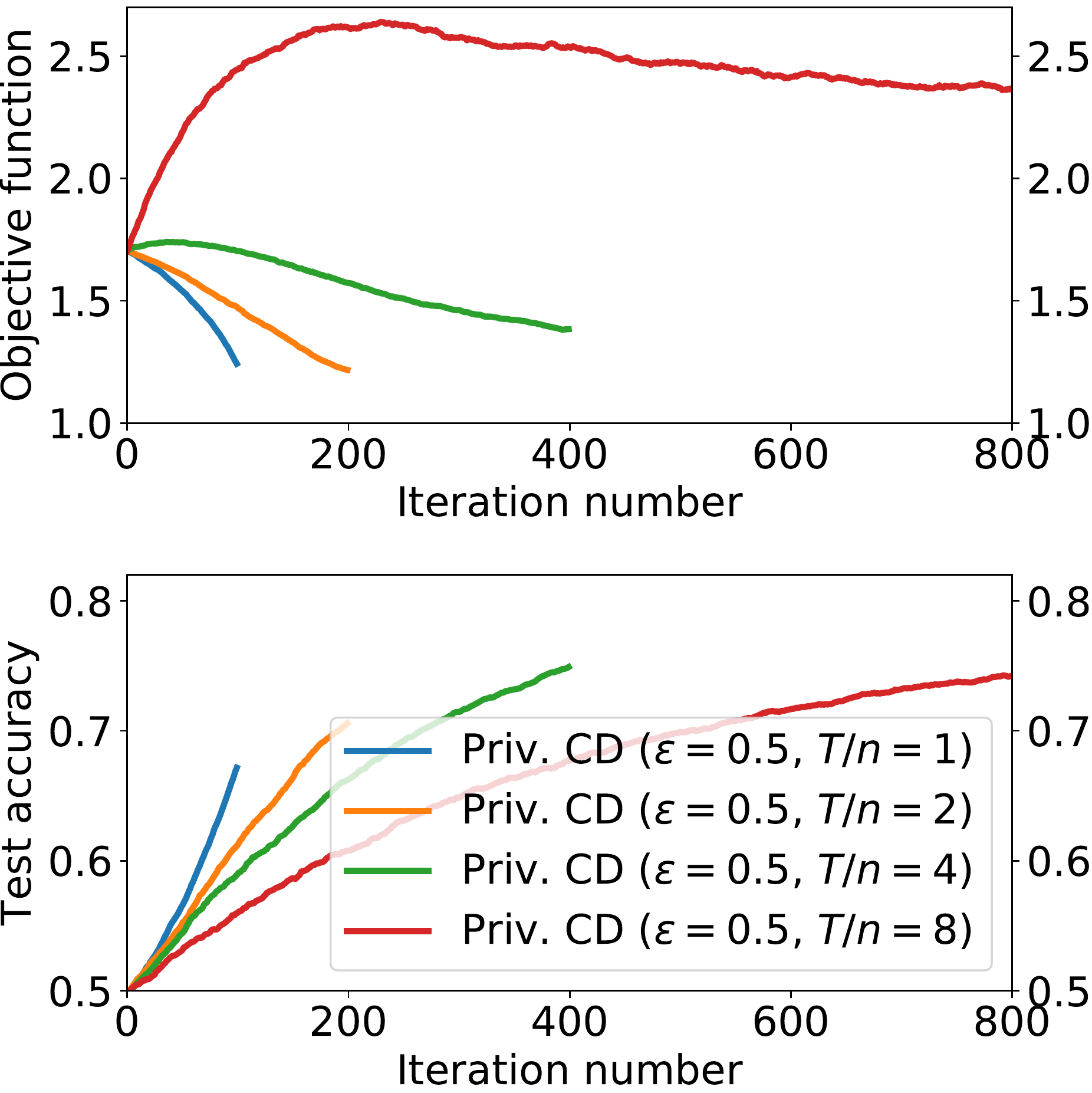}
    }
    \subfigure[\label{fig:private:cdmp}Private init. ($\epsilon=0.05$)]{
    \includegraphics[width=0.27\textwidth]{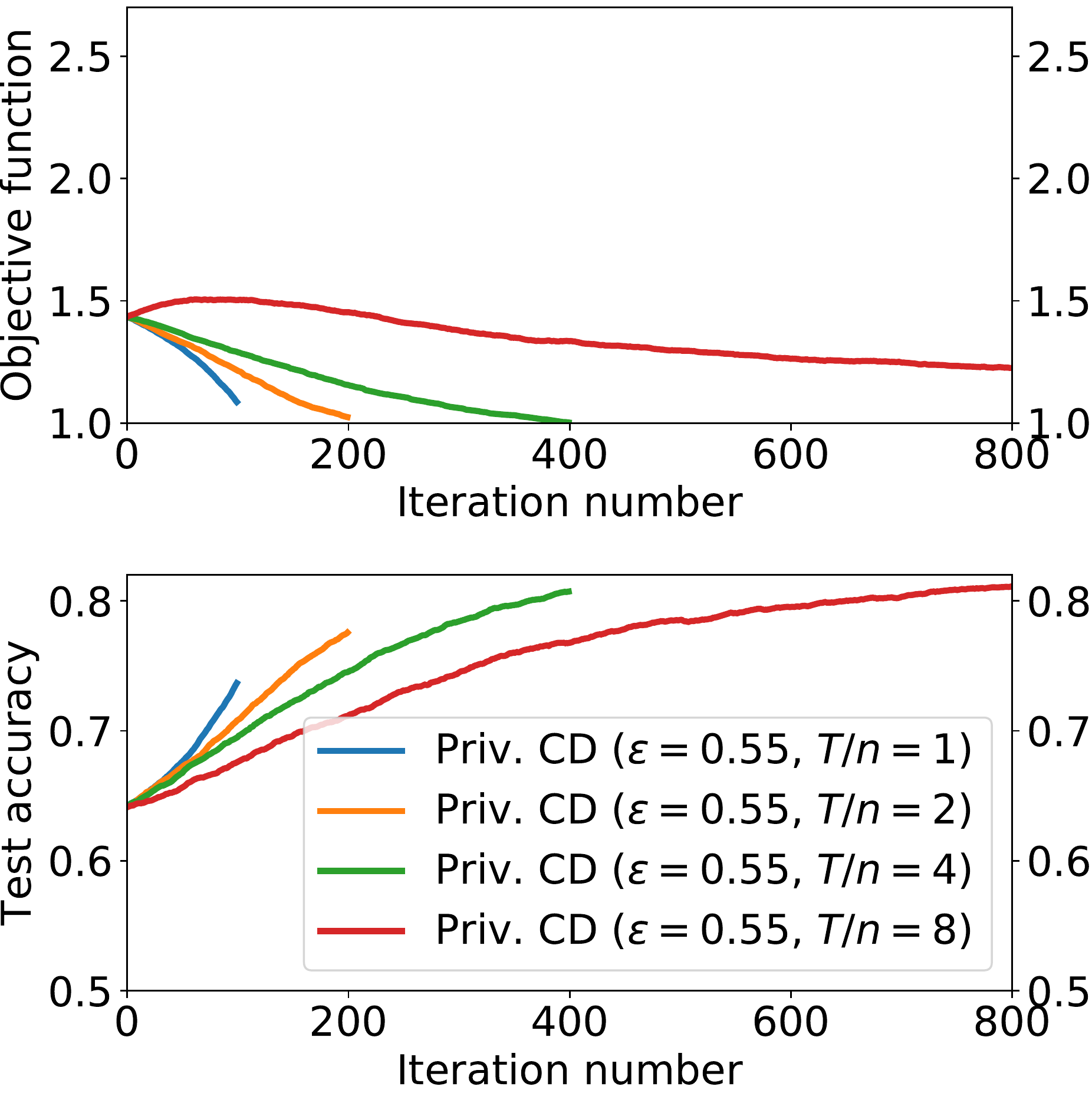}
    }
    \subfigure[\label{fig:private:general}Overall results]{
    \includegraphics[width=0.315\textwidth]{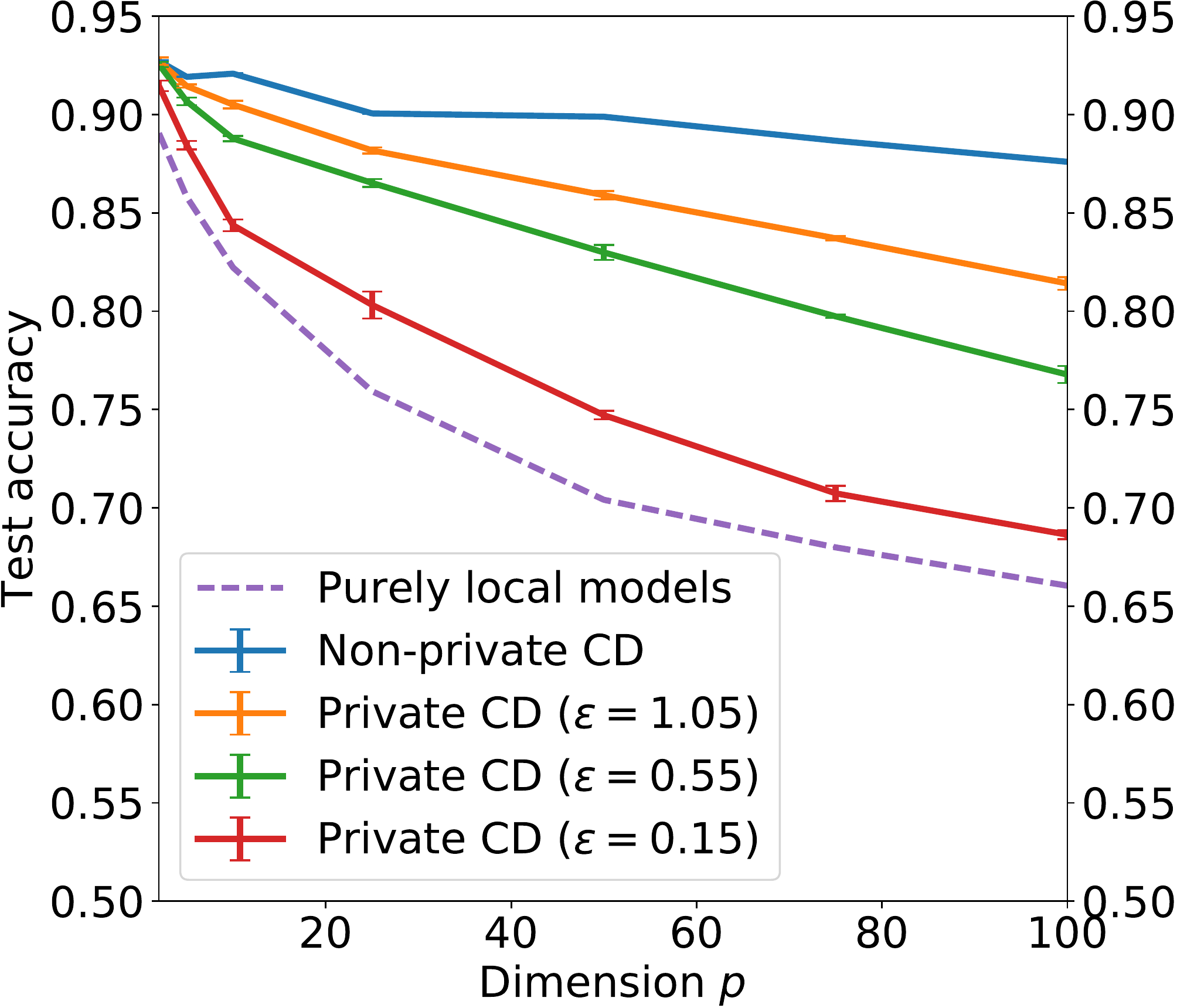}
    }
    \caption{\label{fig:private}Linear classification results in the private setting (averaged over 5 runs). (a)-(b) Evolution of the objective and test accuracy along the iterations for two types of initialization ($p=100$). (c) Final test accuracy for different dimensions and several privacy regimes. Best seen in color.}
\end{figure*}

\section{Related Work}
\label{sec:related}

\textbf{Decentralized ML.}
Most of the work in fully decentralized learning and optimization has focused on the distributed consensus problem, where the goal is to find a single global model which minimizes the sum of the local loss functions \citep{Nedic2009a,Ram2010a,Duchi2012b,Wei2012a,Colin2016a}. In contrast, we tackle the case where agents have distinct objectives.

The work of \citet{Vanhaesebrouck2017a} recently studied the problem of decentralized learning of personalized models and is hence more closely related to our approach, but they did not consider any privacy constraints. At the cost of introducing many auxiliary variables, they cast the objective function as a partial consensus problem over the network which can be solved using a decentralized gossip ADMM algorithm \citep{Wei2013a}.
Our contributions extend over this previous work in several respects: (i) we propose a simpler block coordinate descent algorithm with linear convergence rate, which also proves to be much faster in practice (see Section~\ref{sec:exp}), (ii) we design a differentially private algorithm and provide an analysis of the resulting privacy/utility trade-off, and (iii) we present an evaluation of our approach on real data (in addition to synthetic experiments).



\textbf{DP in distributed learning.}
Differential Privacy has been mostly considered in the context where a ``trusted curator'' has access to all data. Existing DP schemes for learning in this setting typically rely on the addition of noise to the learned model (output perturbation) or to the objective function itself (objective perturbation), see for instance \citet{regularized}.

The private multi-party setting, in which sensitive datasets are distributed across multiple data owners, is known to be harder \citep{McGregor2010a} and has been less studied in spite of its relevance for many applications.
Local DP \citep{Duchi2012b,Kairouz2016a}, consisting in locally perturbing the data points themselves before releasing them, often achieves poor accuracy (especially when local datasets are small). In the master/slave setting, DP algorithms have been introduced to learn a private global model, either by aggregating models trained locally by each party \citep{Pathak2010a,Hamm2016a} or by perturbing the gradients and/or the objective in a distributed gradient descent algorithm \citep{Song2013a,Rajkumar2012a,Shokri2015a}. Some of these approaches rely on the assumption that local datasets are drawn from the same global distribution. The work of \citet{Huang2015a} considers the decentralized setting, using a privacy model similar to ours. However, they still learn a single global model and it is not clear how their algorithm and analysis can be adapted to our multi-task problem. Moreover, their approach is synchronous, relies on additional assumptions (e.g., bounded second derivatives) and does not have established convergence rates.

We conclude this section by briefly mentioning the recent work of \citet{Hitaj2017a} describing an attack against differentially private collaborative deep learning approaches \citep[such as][]{Shokri2015a}. They show how a malicious participant may actively train a Generative Adversarial Network (GAN) which is able to generate prototypical examples of a class held by another agent. While this does not violate DP, it can still constitute a privacy breach in applications where a class distribution itself is considered private.
We believe that the key features of our approach, namely the fully decentralized architecture and the graph regularization over personal models, can significantly limit the effectiveness of the above attack. We leave a careful study of this question for future work.



\begin{table*}[t]
\centering
\begin{tabular}{|r|ccccc|}
\hline
& Purely local models & Non-priv. CD & Priv. $\bar{\epsilon}=1$ & Priv. $\bar{\epsilon}=0.5$ & Priv. $\bar{\epsilon}=0.1$\\
\hline
Per-user test RMSE & 1.2834 & 0.9502 & 0.9527 & 0.9545 & 0.9855\\
\hline
\end{tabular}
\caption{Per-user test RMSE (averaged over users and 5 random runs) on MovieLens-100K.}
\label{tab:ml100k}
\end{table*}

\section{Numerical Experiments}
\label{sec:exp}

\subsection{Linear Classification}

We first conduct experiments on a linear classification task introduced by \citet{Vanhaesebrouck2017a}. We briefly recall the setup. Consider a set of $n=100$ agents. Each of these agents has a target linear separator in $\R^p$ (unknown to the agent). The weight between two agents $i$ and $j$ is given by $W_{ij}=\exp((\cos(\phi_{i,j})-1)/\gamma)$, where $\phi_{i,j}$ is the angle between the target models and $\gamma=0.1$ (negligible weights are ignored).
Each agent $i$ receives a random number $m_i$ of training points (drawn uniformly between $10$ and $100$), where each point is drawn uniformly around the origin and labeled according to the target model. We then add some label noise, independently flipping each label with probability $0.05$. We use the logistic loss $\loss(\theta;x,y)=\log(1+\exp(-y\theta^Tx))$ (which is $1$-Lipschitz), and the L2 regularization parameter of an agent $i$ is set to $\lambda_i=1/m_i>0$ to ensure the overall strong convexity. The hyperparameter $\mu$ is tuned to maximize accuracy of the non-private algorithm on a validation set of random problems instances. For each agent, the test accuracy of a model is estimated on a separate sample of $100$ points.

\textbf{Non-private setting: CD vs ADMM.} We start by comparing our Coordinate Descent (CD) algorithm to the ADMM algorithm proposed by \citet{Vanhaesebrouck2017a} for the non-private setting. Both algorithms are fully decentralized and asynchronous, but our algorithm relies on broadcast (one-way communication from an agent to all neighbors) while the ADMM algorithm is gossip-based (two-way communication between a node and a random neighbor). Which communication model is the most efficient strongly depends on the network infrastructure, but we can meaningfully compare the algorithms by tracking the objective value and the test accuracy with respect to the number of iterations and the number of $p$-dimensional vectors transmitted along the edges of the network. Both algorithms are initialized using the purely local models, i.e. $\Theta_i(0)=\Theta_i^{\loc}$ for all $i\in\intset{n}$.
Figure~\ref{fig:cd_vs_admm} shows the results (averaged over 5 runs) for dimension $p=100$: our coordinate descent algorithm significantly outperforms ADMM despite the fact that ADMM makes several local gradient steps at each iteration ($10$ in this experiment). We believe that this is mostly due to the fact that the 4 auxiliary variables \emph{per edge} needed by ADMM to encode smoothness constraints are updated only when the associated edge is activated. In contrast, our algorithm does not require auxiliary variables.

\textbf{Private setting.}
In this experiment, each agent has the same overall privacy budget $\bar{\epsilon}_i = \bar{\epsilon}$. It splits its privacy budget equally across $T_i=T/n$ iterations using Theorem~\ref{thm:privacy} with $\bar{\delta}_i = \exp(-5)$, and stops updating when it is done.
We first illustrate empirically the trade-offs implied by Theorem~\ref{thm:utility}: namely that running more iterations per agent reduces the first term of the bound but increases the second term because more noise is added at each iteration. This behavior is easily seen in Figure~\ref{fig:private:cdcst}, where $\Theta(0)$ is initialized to a constant vector. 
In Figure~\ref{fig:private:cdmp}, we have initialized the algorithm with a private warm start solution with $\epsilon=0.05$ (see supplementary material). The results confirm that for a modest additional privacy budget, a good warm start can lead to lower values of the objective with less iterations (as suggested again by Theorem~\ref{thm:utility}). The gain in test accuracy here is significant.


Figure~\ref{fig:private:general} shows results for problems of increasing difficulty (by varying the dimension $p$) with various privacy budgets. We have used the same private warm start strategy as in Figure~\ref{fig:private:cdmp}, and the number of iterations per node was tuned based on a validation set of random problems instances.
We see that even under a small privacy budget ($\bar{\epsilon}=0.15$), the resulting models significantly outperform the purely local models (a perfectly private baseline).
As can be seen in the supplementary material, all agents (irrespective of their dataset size) get an improvement in test accuracy. This improvement is especially large for users with small local datasets, effectively correcting for the imbalance in dataset size. We also show that perturbing the data itself, a.k.a. local DP \citep{Duchi2012b,Kairouz2016a}, leads to very inaccurate models. 

\subsection{Recommendation Task}

To illustrate our approach on real-world data, we use MovieLens-100K,\footnote{\url{https://grouplens.org/datasets/movielens/100k/}} a popular benchmark dataset for recommendation systems which consists of 100,000 ratings given by $n=943$ users over a set of $n_{items}=1682$ movies. In our setting, each user $i$ corresponds to an agent who only has access to its own ratings $r_{ij_1},\dots,r_{ij_{m_i}}\in\R$, where $j_1,\dots,j_{m_i}$ denote the indices of movies rated by agent $i$. Note that there is a large imbalance across users: on average, a user has rated 106 movies but the standard deviation is large ($\simeq100$), leading to extreme values (min=20, max=737).
For simplicity, we assume that a common feature representation $\phi_j\in\R^p$ for each movie $j\in\intset{n_{items}}$ is known a priori by all agents ($p=20$ in our experiments). The goal of each agent $i$ is to learn a model $\theta_i\in\R^p$ such that $\theta_i^T\phi_j$ is a good estimate for the rating that $i$ would give to movie $j$, as measured by the quadratic loss $\ell(\theta;\phi,r) = (\theta^T\phi - r)^2$. This is a very simple model: we emphasize that our goal is not to obtain state-of-the-art results on this dataset but to show that our approach can be used to improve upon purely local models in a privacy-preserving manner.
For each agent, we randomly sample 80\% of its ratings to serve as training set and use the remaining 20\% as test set. The network is obtained by setting $W_{ij}=1$ if agent $i$ is within the $10$-nearest neighbors of agent $j$ (or vice versa) according to the cosine similarity between their training ratings, and $W_{ij}=0$ otherwise.
Due to the lack of space, additional details on the experimental setup are deferred to the supplementary material.



Table~\ref{tab:ml100k} shows the test RMSE (averaged over users) for different strategies (for private versions, we use $\bar{\delta}_i = \exp(-5)$ as in the previous experiment). While the purely local models suffer from large error due to data scarcity, our approach can largely outperform this baseline in both the non-private and private settings.


\section{Conclusion}
\label{sec:conclu}

We introduced and analyzed an efficient algorithm for personalized and peer-to-peer machine learning under privacy constraints.
We argue that this problem is becoming more and more relevant as connected objects become ubiquitous. Further research is needed to address dynamic scenarios (agents join/leave during the execution, data is collected on-line, etc.). We will also explore the use of secure multiparty computation and homomorphic encryption as an alternative/complementary approach to DP in order to provide higher accuracy at the cost of more computation.

\paragraph{Acknowledgments} This work was partially supported by grant ANR-16-CE23-0016-01, by a grant from CPER Nord-Pas de Calais/FEDER DATA Advanced data science and technologies 2015-2020 and by European ERC Grant 339539 - AOC (Adversary-Oriented Computing).

\bibliographystyle{apalike}

{\small 
\bibliography{aistats18_CDPrivacy}
}

\appendix

\section*{SUPPLEMENTARY MATERIAL}

This supplementary material is organized as follows. Section~\ref{sec:proofs} contains the proofs of the results in the main text. Section~\ref{sec:utilityanalysis} provides further analysis of Theorem~\ref{thm:utility} for the case where the noise scales are uniform across iterations. Section~\ref{sec:mp} deals with the interesting special case of model propagation and its use as a private warm start strategy. Finally, Section~\ref{sec:addexp} presents additional experimental results and details.

\section{Proofs}
\label{sec:proofs}

\subsection{Proof of Theorem~\ref{thm:privacy}}

We first show that for agent $i$ and an iteration $t_i \in \mathcal{T}_i$, the additional noise $\eta_i(t_i)$ provides $(\epsilon_i(t_i),0)$-differential privacy for the published $\Theta_i(t_i+1)$. In the following, two datasets $\mathcal{S}_i^{1}$ and $\mathcal{S}_i^{2}$ are called neighbors if they differ in a single data point. We denote this neighboring relation by $\mathcal{S}_i^{1}\approx \mathcal{S}_i^{2}$.

We will need the following lemma.

\begin{lemma}
\label{localLossBound}
For two neighboring datasets $\mathcal{S}_i^{1}$ and $\mathcal{S}_i^{2}$ of the size $m_i$:
\begin{align*}
\| \nabla\Loss_i(\Theta_i; \mathcal{S}_i^{1}) - \nabla\Loss_i(\Theta_i; \mathcal{S}_i^{2})\|_{1} \leq \frac{2 \cdot L_0}{m_i}.
\end{align*}
\end{lemma}
\begin{proof}
Assume that instead of data point $(x_1, y_1)$ in $\mathcal{S}_i^{1}$, there is $(x_2, y_2)$ in $\mathcal{S}_i^{2}$. As $\mathcal{S}_i^{1}$ and $\mathcal{S}_i^{2}$ are neighboring datasets, the other data points in $\mathcal{S}_i^{1}$ and $\mathcal{S}_i^{2}$ are the same. Hence:
$$
\| \nabla\Loss_i(\Theta_i; \mathcal{S}_i^{1}) - \nabla\Loss_i(\Theta_i; \mathcal{S}_i^{2})\|_{1} =  \frac{1}{m_i} \| \nabla\loss(\Theta_i; x_1,y_1) - \nabla\loss(\Theta_i; x_2,y_2) \|_{1} \leq \frac{2 \cdot L_0}{m_i},
$$
since the $L_0$-Lipschitzness of $\loss(\cdot; x,y)$ (with respect to the $L_1$-norm) for all $(x,y)\in\mathcal{X}\times\mathcal{Y}$ implies that for any $\Theta_i\in\R^p$ and $(x,y)\in\mathcal{X}\times\mathcal{Y}$, we have $\|\nabla\loss(\Theta_i; x,y)\|_{1}\leq L_0$.
\end{proof}
We continue the proof by bounding the \emph{sensitivity} of $\Theta_{i}(t_i+1)$ to find the noise scale needed to satisfy $(\epsilon,0)$-differential privacy. Using Eq.~\ref{eq:partialgrad}, Eq.~\ref{eq:cdupdate} and Lemma~\ref{localLossBound}, we have:
\begin{align}
sensitivity(\Theta_{i}(t_i+1)) &= \max_{\mathcal{S}_i^{1}\approx \mathcal{S}_i^{2}} \|\Theta_{i}(t_i + 1)\|_1\nonumber\\
 & = \max_{\mathcal{S}_i^{1}\approx \mathcal{S}_i^{2}}\| \frac{1}{L_i}[\nabla\costuni(\Theta(t_i))]_i\|_1\label{eq:sens1}\\
& =\frac{\mu c_i  D_{ii}}{L_i} \max_{\mathcal{S}_i^{1}\approx \mathcal{S}_i^{2}} \| \nabla\Loss_i(\Theta_i; \mathcal{S}_i^{1}) - \nabla\Loss_i(\Theta_i; \mathcal{S}_i^{2})\|_{1}\label{eq:sens2}\\
& \leq \frac{2 \mu c_i  D_{ii} L_0}{m_i L_i},\label{eq:sens3}
\end{align}
where \eqref{eq:sens1}-\eqref{eq:sens2} follow from the fact that $[\nabla\costuni(\Theta(t_i))]_i$ is the only quantity in the update \eqref{eq:cdupdate} which depends on the local dataset of agent $i$.


Recalling the relation between sensitivity and the scale of the addition noise in the context of differential privacy~\citep{sensitivity}, we should have:
$$\epsilon_{i}(t_i) \cdot s_i^* \geq sensitivity(\Theta_{i}(t_i+1)) = \frac{2 \mu c_i  D_{ii} L_0}{m_i L_i},$$
where $s_i^*$ is the scale of the noise added to $\Theta_{i}(t_i+1)$. In the following we show that $s_i^* \geq  \frac{2 \mu c_i  D_{ii} L_0}{m_i L_i \epsilon_{i}(t_i)}$. To compute $s_i^*$, we see how the noise $\eta_i(t_i)$ affects $\widetilde{\Theta}_i(t_i+1) $. Using Eq.~\ref{eq:cdupdatepriv}, definitions of $\alpha$ (Update step page~\pageref{eq:cdupdate}) and $L_i$ (the block Lipschitz constant) we have:
\begin{align*}
\widetilde{\Theta}_i(t_i+1)  &= ( 1-\alpha ) \Theta_i(t) + \alpha \bigg( \sum_{j\in\Nei{i}}\frac{W_{ij}}{D_{ii}}\Theta_j(t) - \mu c_i(\nabla\Loss_i(\Theta_i(t); \mathcal{S}_i)+ \eta_i(t)) \bigg) \\
 & = ( 1-\alpha ) \Theta_i(t) + \alpha \bigg( \sum_{j\in\Nei{i}}\frac{W_{ij}}{D_{ii}}\Theta_j(t) - \mu c_i \nabla\Loss_i(\Theta_i(t); \mathcal{S}_i) \bigg) - \alpha \mu c_i \eta_i(t)\\
 & = \Theta_i(t_i+1) - \frac{\mu c_i \eta_i(t)}{1+\mu c_i L_i^{loc}} \\
 & = \Theta_i(t_i+1)  -\frac{\mu c_i D_{ii} }{L_i} \cdot \eta_i(t_i).
\end{align*}
So the scale of the noise added to $\Theta_i(t_i+1)$ 
is:
$$s_{i}^* = \frac{\mu c_i D_{ii} }{L_i} \cdot s_i(t_i) = \frac{\mu c_i D_{ii} }{L_i} \cdot \frac{2 L_0}{\epsilon_{i}(t_i) m_i} = \frac{2 \mu c_i  D_{ii} L_0}{\epsilon_{i}(t_i) m_i L_i}.$$
Therefore, $s_i^* \geq \frac{sensitivity(\Theta_{i}(t_i+1))}{\epsilon_{i}(t_i) }$ is satisfied, hence publishing $\widetilde{\Theta}_i(t_i+1)$ is $(\epsilon_i(t_i), 0)$-differentially private.

We have shown that at any iteration $t_i \in \mathcal{T}_i$, publishing $\ttheta_i(t_i+1)$ by agent $i$ is $(\epsilon_i(t_i),0)$ differentially private. The mechanism $\mathcal{M}_i$ published all $\ttheta_{i}(t_i+1)$ for $t_i \in \mathcal{T}_i$. Using the composition result for differential privacy established by \citet{composition}, we have that the mechanism $\mathcal{M}_{i}(\mathcal{S}_i)$ is $(\bar{\epsilon}_i,\bar{\delta}_i)$-DP with $\bar{\epsilon}_i,\bar{\delta}_i$ as in Theorem~\ref{thm:privacy}.

Theorem~\ref{thm:privacy} considers the case where $\loss(\cdot; x,y)$ is $L_0$-Lipschitz for all $(x,y)\in\mathcal{X}\times\mathcal{Y}$ with respect to the $L_1$-norm. We could instead assume Lipschitzness with respect to the $L_2$-norm, in which case the noise to add should be Gaussian instead of Laplace. The following remark computes the additional normal noise to preserve differential privacy in this setting.
\begin{remark}
Let $i\in\intset{n}$.
In the case where $\loss(\cdot; x,y)$ is $L^*_0$-Lipschitz with respect to the $L_2$-norm for all $(x,y)\in\mathcal{X}\times\mathcal{Y}$, for any $t_i\in\mathcal{T}_i$, let $s_i(t_i) \geq 2 L^*_0 \sqrt{2 \ln (2/\delta_i(t_i))}/\epsilon_i(t_i)$ for some $\epsilon_i(t_i)>0$ and $\delta_i(t_i) \in [0,1]$.
For the noise vector $\eta_i(t)$ drawn from a Gaussian distribution $\mathcal{N}(0,s_i(t))^p\in\R^p$ with scale $s_i(t_i)$,
and for any $\bar{\delta}_i\in[0,1]$ and initial point $\widetilde{\Theta}(0)\in\R^{np}$ independent of $\mathcal{S}_i$, the mechanism $\mathcal{M}_{i}(\mathcal{S}_i)$ is $(\bar{\epsilon}_i,1-(1-\bar{\delta}_i)\prod_{t_i = 1}^{T_i}(1-\delta_i(t_i)))$-DP with
\begin{align*}
\bar{\epsilon}_i = \min\Bigg\{ & \sum_{t_i = 1}^{T_i} \epsilon_{i}(t_i), \sum_{t_i = 1}^{T_i}{\frac{(e^{\epsilon_{i}(t_i)}-1)\epsilon_{i}(t_i)}{e^{\epsilon_{i}(t_i)}+1}+\sqrt{\sum_{t_i = 1}^{T_i} 2\epsilon_{i}(t_i)^2\log\Big(e+\frac{\sqrt{\sum_{t_i = 1}^{T_i}\epsilon_{i}(t_i)^{2}}}{\bar{\delta_{i}}}\Big)}},\\
& \sum_{t_i = 1}^{T_i}{\frac{(e^{\epsilon_{i}(t_i)}-1)\epsilon_{i}(t_i)}{e^{\epsilon_{i}(t_i)}+1}+\sqrt{\sum_{t_i = 1}^{T_i} 2\epsilon_{i}(t_i)^2\log(1/\bar{\delta_{i}})}}\Bigg\}.
\end{align*}
\end{remark}

\subsection{Proof of Theorem~\ref{thm:utility}}

We start by introducing a convenient lemma.

\begin{lemma}
\label{lem:lip1}
For any $i\in[n]$, $\Theta\in\R^{np}$ and $d\in\R^p$ we have:
$$\costuni(\Theta+U_id) \leq \costuni(\Theta) + d^T[\nabla\costuni(\Theta)]_i + \frac{L_i}{2}\|d\|^2.$$
\end{lemma}
\begin{proof}
We get this by applying Taylor's inequality to the function
\begin{eqnarray*}
q_\Theta^x &:& \R^p \rightarrow \R\\
&& d \mapsto \costuni(\Theta+U_id).
\end{eqnarray*}
\end{proof}

Recall that the random variable $\eta_{i}(t)\in\R^p$ represents the noise added by agent $i\in\intset{n}$ due to privacy requirements if it wakes up at iteration $t\geq 0$. To simplify notations we denote the scaled version of the noise by $\widetilde{\eta}_{i}(t) = \mu D_{ii} c_i \eta_{i}(t)$.

Let $i_t$ be the agent waking up at iteration $t$. Using Lemma~\ref{lem:lip1}, we have:
\begin{align*}
\costuni(\ttheta(t+1)) &= \costuni\left(\ttheta(t)-\frac{U_{i_t}}{L_{i_t}}\left([\nabla\costuni(\ttheta(t))]_{i_t} + \widetilde{\eta}_{i_t}(t)\right)\right)\\
&\leq \costuni(\ttheta(t)) - \frac{1}{L_{i_t}}[\nabla\costuni(\ttheta(t))]_{i_t}^T\left([\nabla\costuni(\ttheta(t))]_{i_t} + \widetilde{\eta}_{i_t}(t)\right)\\
&+ \frac{1}{2L_{i_t}}\|[\nabla\costuni(\ttheta(t))]_{i_t} + \widetilde{\eta}_{i_t}(t)\|^2\\
&= \costuni(\ttheta(t)) - \frac{1}{L_{i_t}}\|[\nabla\costuni(\ttheta(t))]_{i_t}\|^2 + \frac{1}{2L_{i_t}}\|[\nabla\costuni(\ttheta(t))]_{i_t}\|^2+ \frac{1}{2L_{i_t}}\|\widetilde{\eta}_{i_t}(t)\|^2\\
&\leq \costuni(\ttheta(t)) - \frac{1}{2L_{max}}\|[\nabla\costuni(\ttheta(t))]_{i_t}\|^2 + \frac{1}{2L_{min}}\|\widetilde{\eta}_{i_t}(t)\|^2,
\end{align*}
where $L_{min} = \min_{1\leq i \leq n} L_i$.

Recall that under our Poisson clock assumption, each agent is equally likely to wake up at any step $t$.
Subtracting $\costuni^*$ and taking the expectation with respect to $i_t$ on both sides, we thus get:
\begin{align}
\myexp_{i_t}[\costuni(\ttheta(t+1))] - \costuni^* &\leq \costuni(\ttheta(t)) - \costuni^* - \frac{1}{2L_{max}}\frac{1}{n}\sum_{j=1}^n\|[\nabla\costuni(\ttheta(t))]_{j}\|^2\nonumber\\
&+ \frac{1}{2L_{min}}\frac{1}{n}\sum_{j=1}^n\|\widetilde{\eta}_{j}(t)\|^2\nonumber\\
&= \costuni(\ttheta(t)) - \costuni^* - \frac{1}{2nL_{max}}\|\nabla\costuni(\ttheta(t))\|^2 + \frac{1}{2nL_{min}}\|\widetilde{\eta}(t)\|^2,\label{eq:proof2}
\end{align}
where $\widetilde{\eta}(t)=[\widetilde{\eta}_1(t);\dots,;\widetilde{\eta}_n(t)]\in\R^{np}$.

For convenience, let us define $P_t = \myexp[\costuni(\ttheta(t))] - \costuni^*$ where $\myexp[\cdot]$ denotes the expectation with respect to all variables $\{i_t\}_{t\geq 0}$ and $\{\eta(t)\}_{t\geq 0}$. Using \eqref{eq:proof2} we thus have:
\begin{align}
P_{t+1} &\leq P_t - \frac{1}{2nL_{max}}\myexp[\|\nabla\costuni(\ttheta(t))\|^2] + \frac{1}{2nL_{min}}\myexp[\|\widetilde{\eta}(t)\|^2]\label{eq:recur2}
\end{align}

Recall that $\costuni$ is $\sigma$-strongly convex, i.e. for any $\Theta,\Theta'\in\mathbb{R}^{np}$ we have:
$$\costuni(\Theta') \geq \costuni(\Theta) + \nabla\costuni(\Theta)^T(\Theta'-\Theta) + \frac{\sigma}{2}\|\Theta'-\Theta\|_2^2.$$
We minimize the above inequality on both sides with respect to $\Theta'$. We obtain that $\Theta'=\Theta^*$ minimizes the left-hand side, while $\Theta'=\Theta - \nabla\costuni(\Theta) / \sigma$. We thus have for any $\Theta$:
$$\costuni^* \geq \costuni(\Theta) -\frac{1}{\sigma} \nabla\costuni(\Theta)^T\nabla\costuni(\Theta) + \frac{1}{2\sigma}\|\nabla\costuni(\Theta)\|_2^2 = \costuni(\Theta) -\frac{1}{2\sigma} \|\nabla\costuni(\Theta)\|^2.$$

Using the above inequality to bound $\|\nabla\costuni(\ttheta(t))\|^2$ in \eqref{eq:recur2}, we obtain:
$$P_{t+1} \leq P_t - \frac{\sigma}{nL_{max}}P_t + \frac{1}{2nL_{min}}\myexp[\|\widetilde{\eta}(t)\|^2] = \left(1 - \frac{\sigma}{nL_{max}}\right)P_t +\frac{1}{2nL_{min}}\myexp[\|\widetilde{\eta}(t)\|^2].$$
A simple recursion on $P_t$ gives:
\begin{equation}
\label{eq:util}
P_t \leq \left( 1 - \frac{\sigma}{nL_{max}} \right)^t P_0 + \frac{1}{2nL_{min}}\sum_{t'=0}^{t-1}\left( 1 - \frac{\sigma}{nL_{max}} \right)^{t'}\myexp[\|\widetilde{\eta}(t')\|^2].
\end{equation}

For any $t\geq 0$ and $i\in\intset{n}$, the entries of $\eta_i(t)$ are drawn from independent Laplace distributions with mean $0$ and scale $s_i(t)$, hence we have:
$$\myexp[\|\widetilde{\eta}(t)\|^2] = \myexp\left[\sum_{i=1}^n\|\mu D_{ii} c_i\eta_i(t)\|^2\right] = \sum_{i=1}^n 2(\mu D_{ii} c_is_i(t))^2.$$
This concludes the proof.

\subsection{Proof of Proposition~\ref{prop:allfullprivate}}

We start the proof with the following lemma.
\begin{lemma}
\label{cor:allocation}
Let $C=1 - \sigma/nL_{max}$. Assume that for any $i\in\intset{n}$, we have $s_i(t_i)=\frac{2 L_0}{\epsilon_i(t_i) m_i}$ for $t_i\in\mathcal{T}_i$ as in Theorem~\ref{thm:privacy}. Given a total number of iterations $T$ and the overall privacy budgets $\bar{\epsilon}_1,\dots,\bar{\epsilon}_n>0$ of each agent, the following privacy parameters minimize the utility loss:
$$
\epsilon_i^*(t) = \frac{\sqrt[3]{ C} - 1}{\sqrt[3]{C^T} - 1} \sqrt[3]{C^t} \bar{\epsilon}_i,\quad\forall i\in\intset{n}, 0\leq t \leq T-1.
$$
\end{lemma}
\begin{proof}
Denote $A_i = \frac{\mu c_i D_{ii} L_0}{m_i}$. 
As the first part of the upper-bound of the utility loss in Theorem~\ref{thm:utility} does not depend on the $\epsilon_i(t)$'s, we need to find the $\epsilon_i(t)$'s which minimize the following quantity:
\begin{align}
\min \sum_{t=0}^{T-1}\sum_{i=1}^n\Big( 1 - \frac{\sigma}{nL_{max}} \Big)^{t}\big(\mu c_i D_{ii}s_i(t)\big)^2 & = 
\min \sum_{t=0}^{T-1}\sum_{i=1}^n\Big( 1 - \frac{\sigma}{nL_{max}} \Big)^{t}\Big(\frac{2\mu c_i D_{ii} L_0}{\epsilon_i(t_i) m_i}\Big)^2 \nonumber\\ 
 & = \min \sum_{t=0}^{T-1} C^{t} \Big(\sum_{i = 1}^{n} \frac{A_i^2}{\epsilon_i^2(t)}\Big)\nonumber \\
 & = \min \sum_{i = 1}^{n} \Big(\sum_{t=0}^{T-1} \frac{A_i^2 C^{t}}{\epsilon_i^2(t)}\Big),\label{objective}
\end{align}
under the constraints: $\forall i, \epsilon_i(t) \geq 0$ and $\sum_{t=0}^{T-1} \epsilon_i(t) = \bar{\epsilon}_i$.
As the agents are independent, we can solve the above optimization problem separately for each agent $i$ to minimize $\sum_{t=0}^{T-1} \frac{A_i^2 C^{t}}{(\epsilon_i(t))^2}$ under the constraint $\sum_{t=0}^{T-1} \epsilon_i(t) = \bar{\epsilon}_i$. Denote $\epsilon_i(t) = x_{it}$. We have $x_{i0} = \bar{\epsilon}_i - \sum_{t = 1}^{T-1}x_{it}$. Replacing $x_{i0}$ in the objective function of Eq.~\ref{objective}, we can write the objective function as follows:
$$
\forall i: \; \; F_i(x_{i1}, \cdots, x_{iT-1}) = \frac{A_i^2}{(\bar{\epsilon}_i - \sum_{t = 1}^{T - 1}{x_{it}})^2} + \sum_{t = 1}^{T -1} {\frac{A_i^2 C^t}{x_{it}^2}}.
$$
The problem is thus equivalent to finding $\min F_i(x_{i1}, \cdots, x_{it-1}) \; \; \forall i$ under the previous constraints.
To find the optimal $x_{it}$ we find set its partial derivative $\frac{\partial F_i}{\partial x_{it}}$ to $0$:
$$
\frac{\partial F_i}{\partial x_{it}} = \frac{-2A_i^2 C^t}{x_{it}^3} + \frac{2 A_i^2}{(\bar{\epsilon}_i - \sum_{t = 1}^{T-1}x_{it})^3} = 0.
$$
Hence the value of $x_{it}$ satisfies:
\begin{align}
\label{each}
 \forall t \in [1:T-1]: x_{it} = \sqrt[3] {C^t} \Big(\bar{\epsilon}_i - \sum_{j = 1}^{T-1}x_{ij}\Big).
\end{align}
For a fixed $i$, after summing up all $x_{it}$s, we get:
$$
\sum_{j = 1}^{T-1}x_{ij} = \Big(\sum_{j = 1}^{T-1} \sqrt[3]{C^j}\Big)\Big(\bar{\epsilon}_i - \sum_{j = 1}^{T-1}x_{ij}\Big) = \frac{\sqrt[3]{C^T} - \sqrt[3]{C}}{\sqrt[3]{C} - 1}\Big(\bar{\epsilon}_i - \sum_{j = 1}^{T-1}x_{ij}\Big).
$$
And hence:
$$
x_{i0} = \bar{\epsilon}_i - \sum_{j = 1}^{T-1}x_{ij} = \frac{\sqrt[3]{C} - 1}{\sqrt[3]{C^T} - 1}\bar{\epsilon}_i.
$$
Using Eq.~\ref{each}, we thus finally get:
$$
\forall i,t: \epsilon_i^*(t) = \frac{\sqrt[3]{C} - 1}{\sqrt[3]{C^T} - 1} \sqrt[3]{C^t} \bar{\epsilon}_i
$$
\end{proof}

The noise allocation strategy given by the above lemma optimizes the utility loss, which is an expectation with respect to the clock ticks. Hence $\epsilon_i^*(t)$ gives the amount of noise that an agent should add \emph{in expectation} at a given global iteration $t$. Note that we have $\sum_{t=0}^{T-1}\epsilon_i^*(t) = \bar{\epsilon}_i$, hence the budget is matched exactly in expectation.
For a particular run of the algorithm however, each agent $i$ only wakes up at a subset $\mathcal{T}_i$ of all iterations and only uses its privacy budget at these iterations $t\in\mathcal{T}_i$. Thus, for a single run, the above strategy does not use up the entire privacy budget $\bar{\epsilon}_i$.

If we instead condition on a particular schedule, i.e. each agent $i$ knows $\mathcal{T}_i$ in advance, we can appropriately renormalize the privacy parameters to make sure that the agents truly utilize their entire privacy budget.
The overall privacy parameter for agent $i$ with optimal noise allocation as in Lemma~\ref{cor:allocation} given $\mathcal{T}_i$ is as follows:
$$
\sum_{t \in \mathcal{T}_i} \epsilon_i^*(t) = \lambda_{\mathcal{T}_i}(i) \cdot \bar{\epsilon}_i.
$$
Hence we can set
$$
\sum_{t \in \mathcal{T}_i} \epsilon_i(t) = \sum_{t \in \mathcal{T}_i} \frac{\epsilon_i^*(t)}{\lambda_{\mathcal{T}_i}(i)} =  \bar{\epsilon}_i.
$$
This shows that the privacy parameters defined in Proposition~\ref{prop:allfullprivate} fulfills the overall privacy budget of an agent $i$ during the iterations that agent $i$ wakes up.
Furthermore, for $\epsilon_i^*(t)$ defined in Lemma~\ref{cor:allocation}, $\sum_{t = 0}^{T - 1} \epsilon_i^*(t) = \bar{\epsilon}_i\quad\forall i\in\intset{n}$.
We then conclude that $\lambda_{\mathcal{T}_i}(i) \leq 1$ as $\mathcal{T}_i \subseteq \intset{T}$.
The additional noise to provide $\epsilon_i(t)$ differential privacy (as in Proposition~3) is thus lower than the additional noise to provide $\epsilon_i^*(t)$ differential privacy in Lemma~\ref{cor:allocation}. 

\section{Further Analysis of Theorem~\ref{thm:utility}}
\label{sec:utilityanalysis}

We further analyze the utility/privacy trade-off of Theorem~\ref{thm:utility} in the special case where the noise scales are uniform across iterations, i.e. $s_i(t)=s_i$ for all $t$ and $i$. Let us denote
$$a = \frac{1}{nL_{min}}\sum_{i=1}^n(\mu D_{ii}c_is_i)^2,\quad b = \costuni(\Theta(0)) - \costuni^*,\quad\rho = \frac{\sigma}{nL_{max}}.$$

The additive error due to the noise is a sum of a geometric series and thus simplifies to
$$\sum_{t=0}^{T-1}a(1-\rho)^t = a\left(\frac{1-(1-\rho)^T}{\rho}\right).$$

Hence we can write the inequality of Theorem~\ref{thm:utility} as
\begin{equation*}
\mathbb{E}\left[\costuni(\Theta(T)) - \costuni^\star\right] \leq \underbrace{b(1-\rho)^T}_{\text{optimization error}} + \underbrace{\frac{a}{\rho}\left(1-(1-\rho)^T\right)}_{\text{noise error}}.
\end{equation*}

The optimization error after $T$ iterations thus decomposes into two terms: an optimization error term (which is the same quantity that one gets in the non-private setting) and a noise error term due to privacy. One can see that if the algorithm is run until convergence ($T\rightarrow\infty$), the (additive) utility loss due to privacy is given by
$$\frac{a}{\rho}=\frac{L_{max}}{\sigma L_{min}}\sum_{i=1}^n(\mu D_{ii}c_is_i)^2.$$
However, this additive loss due to noise can actually be smaller than $\frac{a}{\rho}$ for $T<\infty$. Indeed, increasing $T$ from 0 to $\infty$ makes the noise error go from $0$ to $\frac{a}{\rho}$, but also drives the optimization error from $b$ to $0$. This suggests that the number of iterations $T$ can be used to minimize the overall error and should be carefully tuned. This is confirmed by our numerical experiments in Figures~\ref{fig:private:cdcst}-\ref{fig:private:cdmp}.

\section{Propagation of (Private) Local Models}
\label{sec:mp}

In this section, we give some details on an interesting special case of our framework. The idea is to smooth models that are pre-trained locally by each agent. Formally, the objective function to minimize is as follows:
\begin{equation}
    \costmp(\Theta) = \frac{1}{2} \bigg( \sum_{i<j}^n W_{ij} {\lVert \Theta_i - \Theta_j \rVert}^2 + \mu \sum_{i=1}^n D_{ii} c_i {\lVert \Theta_i - \Theta_i^{\loc} \rVert}^2 \bigg),
    \label{eq:Qmp}
\end{equation}
which is a special case of \eqref{eq:Qunified} when we set $\Loss_i(\Theta_i; \mathcal{S}_i) = \frac{1}{2} {\lVert \Theta_i - \Theta_i^{\loc} \rVert}^2$ with $\Theta_i^{\loc} \in\argmin_{\theta\in\mathbb{R}^p} \sum_{j=1}^{m_i} \loss(\theta; x_i^j,y_i^j)+\lambda_i\|\theta\|^2$.
Each $\Loss_i$ is $1$-strongly convex in $\Theta_i$ with $1$-Lipschitz continuous gradient, hence we have $L_i=D_{ii}(1+\mu c_i)$ and $\sigma\geq \mu\min_{1\leq i\leq n}[D_{ii}c_i]$.

Developing the update step \eqref{eq:cdupdate} for this special case, we get:
\begin{equation}
\label{eq:mpupdate}
\Theta_i(t+1) = \frac{1}{1+\mu c_i}\bigg( \sum_{j\in\Nei{i}}\frac{W_{ij}}{D_{ii}}\Theta_j(t) + \mu c_i \Theta_i^{\loc}\bigg).
\end{equation}
Because the objective function \eqref{eq:Qmp} is quadratic and separable along the $p$ dimensions, \eqref{eq:mpupdate} corresponds to the \emph{exact minimizer} of $\costmp$ along the block coordinate direction $\Theta_i$. Hence $\Theta_i(t+1)$ does not depend on $\Theta_i(t)$, but only on the solitary and neighboring models.

It turns out that we recover the update rule proposed specifically by \citet{Vanhaesebrouck2017a} for model propagation. Our block coordinate algorithm thus generalizes their approach to general local loss functions and allows to obtain convergence rate instead of only asymptotic convergence.

\paragraph{Private setting}
In the above objective, the interaction with each local dataset $\mathcal{S}_i$ is only through the pre-trained model $\Theta_i^{\loc}$ learned by minimizing the (regularized) empirical risk as denoted in \eqref{eq:solitary}. Therefore, if we generate a DP version $\widetilde{\Theta}_i^{\loc}$ of $\Theta_i^{\loc}$, we can run the non-private coordinate descent algorithm \eqref{eq:mpupdate} using $\widetilde{\Theta}_i^{\loc}$ instead of $\Theta_i^{\loc}$ without degrading the privacy guarantee. We can thus avoid the dependency on the number of iterations of the coordinate descent algorithm.
Several well-documented methods for an agent to generate a DP version of its local model $\Theta_i^{\loc}$ exist in the literature, under mild assumptions on the loss function and regularizer. One may for instance use output or objective perturbation \citep{regularized}.

Recall that Theorem~\ref{thm:utility} emphasizes that it is beneficial to have a good warm start point $\Theta(0)$ for the general algorithm: the smaller $\costuni(\Theta(0)) - \costuni^\star$, the less iterations needed to decrease the optimization error to the desired precision and hence the less noise added due to privacy. However, to ensure the overall privacy of the algorithm, $\Theta(0)$ must be also be differentially private. In light of the discussion above, we can use the private model propagation solution as a good private warm start. 


\section{Additional Experimental Results and Details}
\label{sec:addexp}

\subsection{Linear classification}

\begin{figure}[t]
    \centering
    \includegraphics[width=0.75\textwidth]{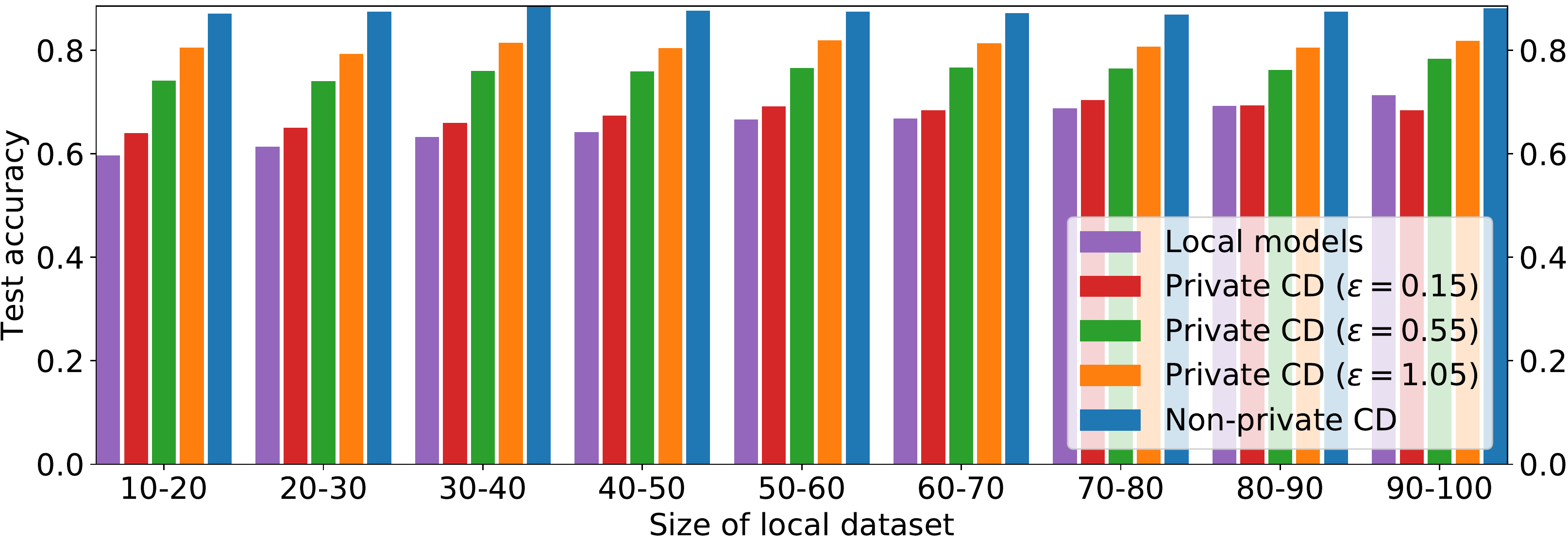}
    \caption{\label{fig:cd_hist} Final test accuracy (averaged over 5 runs) per local dataset size for dimension $p=100$ and several privacy regimes for our private and non-private coordinate descent algorithms. Agents with small datasets get larger improvements. Best seen in color.}
\end{figure}

We present here some additional experimental results on the linear classification task to complement those displayed in the main text.

\paragraph{Test accuracy w.r.t. size of local dataset}
Figure~\ref{fig:cd_hist} shows the test accuracy of our algorithm under different private regimes depending on the size of the local dataset of an agent.
First, we can see that all agents (irrespective of their dataset size) benefit from private collaborative learning, in the sense that their final accuracy is larger than that of their purely local model. This is important as it gives an incentive to all agents (including those with larger datasets) to collaborate.
Second, the algorithm effectively corrects for the imbalance in dataset size: agents with less data generally get a larger boost in accuracy and can almost catch up with better-endowed agents (which also get an improvement, albeit smaller).

\begin{figure}[t]
    \centering
    \includegraphics[width=0.55\textwidth]{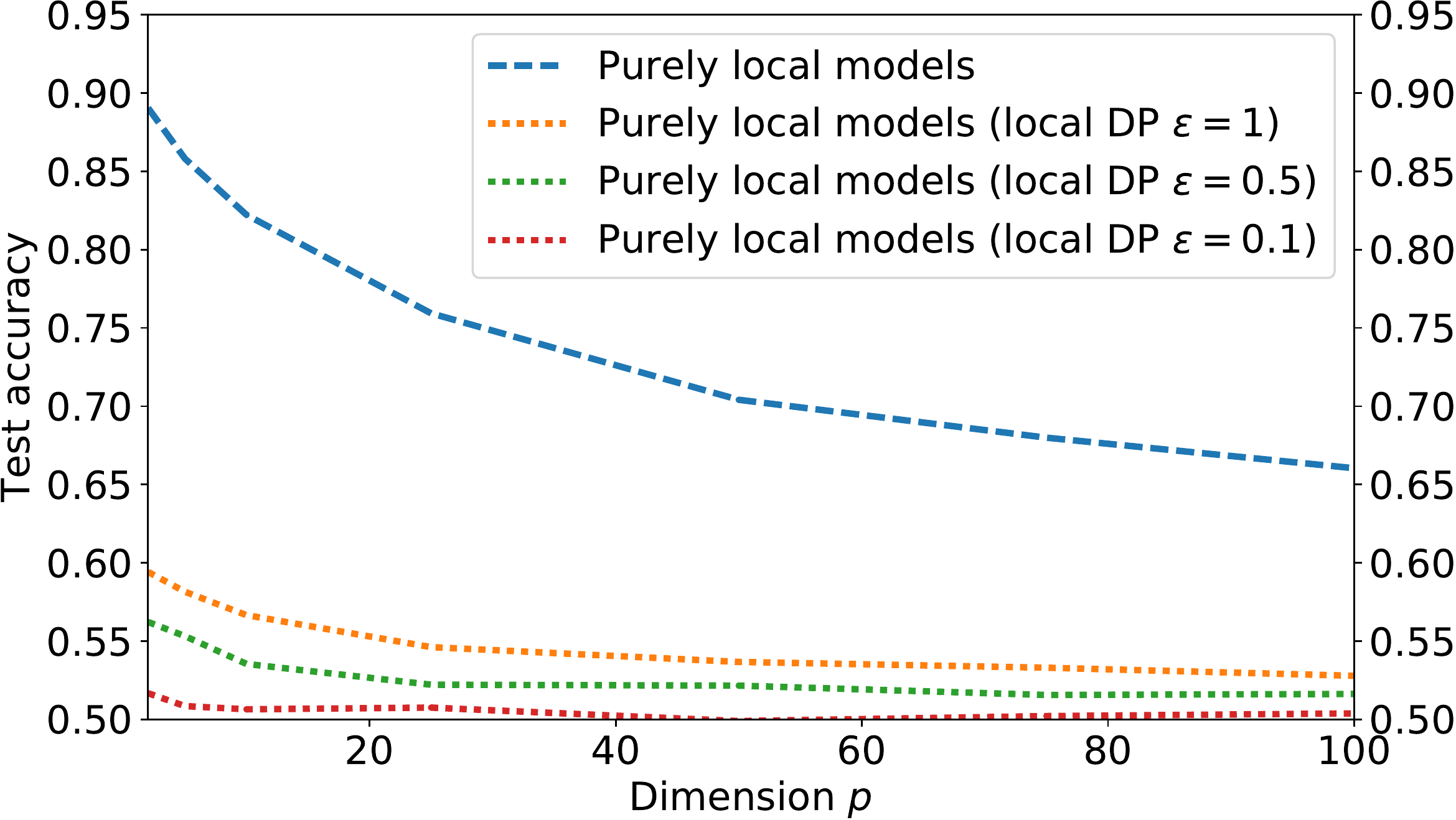}
    \caption{\label{fig:localdp} Test accuracy (averaged over 5 runs) of purely local models learned from perturbed data (local DP) w.r.t. to data dimension, for different privacy regimes. Best seen in color.}
\end{figure}

\paragraph{Local Differential Privacy baseline}
As mentioned in Section~\ref{sec:related}, local Differential Privacy \citep{Duchi2012b,Kairouz2016a} consists in adding noise to each data point itself (proportional to the sensitivity of its features) so that the resulting perturbed point does not reveal too much about the original point. Local DP can be used trivially in the distributed/decentralized setting as it is purely local. However it is also a very conservative approach as it is agnostic to the type of analysis done on the data (the perturbed dataset can be released publicly).
Figure~\ref{fig:localdp} shows the accuracy of purely local models learned after applying $(\epsilon, 0)$-local DP to each local dataset. As expected, the loss in accuracy compared to purely local models learned from unperturbed data is huge. We were not able to improve significantly over these models by applying our collaborative learning algorithm on the perturbed data, confirming that it contains mostly random noise. This confirms the relevance of our private algorithm based on perturbing the updates rather than the data itself.

\subsection{Recommendation Task}

We provide here some details on our experimental setup for MovieLens-100K:
\begin{itemize}
\item \textbf{Normalization}: Following standard practice, we normalize the data user-wise by subtracting from each rating the average rating of the associated user.
\item \textbf{Movie features}: In principle, for privacy reasons, the features should be computed independently of the training data (e.g., computed from a separate set of user ratings, or descriptive features obtained from IMDB). For convenience, we simply use the movie features generated by a classic alternating least square method for collaborative filtering \citep{Zhou2008a} applied to our training set (i.e., random 80\% of the ratings). We set the feature dimensionality $p$ to $20$ as it is typically enough to reach good performance on MovieLens-100K.
\item \textbf{Lipschitz constant}: One can bound the Lipschitz constant of the quadratic loss by assuming a bound on the norm of model parameters and movie features. However, this often results in overestimating the Lipschitz constant and hence in an unnecessarily large scale for the added noise used to enforce differential privacy. In such cases it is easier to clip the norm of all point-wise gradients when they exceed a constant $C$ \citep[see for instance][]{Abadi2016a}. This ensures that \eqref{eq:sens3} holds (replacing $L_0$ by $C$). We set $C=10$, which is large enough that gradients are almost never clipped in practice.
\item \textbf{Hyperparameters}: We simply set the L2 regularization parameter of each agent $i$ to $\lambda_i=1/m_i$, and use $\mu=0.04$ as in the previous synthetic experiment. For our private algorithm, the number of iterations per node is tuned for each value of $\bar{\epsilon}$ on a validation set (obtained by a 80/20 random split of the training set of each agent).
\end{itemize}

\end{document}